\documentclass{article}

\usepackage{times}
\usepackage{soul}
\usepackage{url}
\usepackage[hidelinks]{hyperref}
\usepackage[utf8]{inputenc}
\usepackage[small]{caption}
\usepackage{booktabs}
\usepackage{algorithm}
\usepackage{algorithmic}
\usepackage[switch]{lineno}
\usepackage{graphicx} 
\usepackage{amsmath,amsthm} 
\usepackage{amssymb}
\usepackage{xcolor}
\usepackage{tikz}
\usepackage{placeins}  
\usepackage{multirow}  
\usepackage{hhline} 
\usepackage{caption} 
\usepackage{hypcap} 
\usepackage{authblk}

\newcommand{\precdot}{\prec\mathrel{\mkern-5mu}\mathrel{\cdot}}

\newtheorem{theorem}{Theorem}

\definecolor{orange}{RGB}{255,127,0}
\definecolor{brown}{RGB}{150,70,0}
\definecolor{green}{RGB}{127,255,127}
\definecolor{darkgreen}{RGB}{0,127,0}
\definecolor{blue}{RGB}{127,127,255}
\definecolor{lightblue}{RGB}{150,150,255}
\definecolor{darkblue}{RGB}{0,0,127}
\definecolor{red}{RGB}{255,90,90}
\definecolor{violet}{RGB}{200,110,170}
\definecolor{grey}{RGB}{127,127,127}
\definecolor{pink}{RGB}{255,180,180}

\makeatletter
\newcommand\thefontsize[1]{{#1 The current font size is: \f@size pt\par}}
\makeatother

\newcommand{\eager}{Eager}
\newcommand{\greedy}{Greedy}

\newcommand{\opta}{Optimal-1 }
\newcommand{\optb}{Optimal-2 }

\title{When Redundancy Matters:\\ Machine Teaching of Representations\footnote{This project is supported by NRC project (329745) Machine Teaching for Explainable AI
 }}

\author[1]{Cèsar Ferri}
\author[2]{Dario Garigliotti}
\author[2]{Brigt Arve Toppe Håvardstun}
\author[1]{Josè Hernández-Orallo}
\author[2]{Jan Arne Telle}
\affil[1]{ VRAIN, Universitat Politècnica de València, València, Spain}
\affil[2]{Department of Informatics, University of Bergen, Norway}
\date{June 2023}

\begin{document}
\date{}
\maketitle

\begin{abstract}
  In traditional machine teaching, a teacher wants to teach a concept to a learner,  by means of a finite set of examples, the witness set. But concepts 
can have many equivalent representations. This redundancy strongly affects the search space, to the extent that  teacher and learner may not be able to easily determine the equivalence class of each representation. In this common situation, instead of teaching concepts, we explore the idea of teaching representations. 
We work with several teaching schemas that exploit representation and witness {\em size} (Eager, Greedy and Optimal) and analyze the gains in teaching effectiveness for some representational languages (DNF expressions and Turing-complete P3 programs). Our theoretical and experimental results indicate that there are various types of redundancy, handled better by the Greedy schema introduced here than by the Eager schema, although both can be arbitrarily far away from the Optimal.  
For P3 programs we found that witness sets are usually smaller than the programs they identify, which is an illuminating justification of why machine teaching from examples makes sense at all.
\end{abstract}

\section{Introduction}
\label{sec-intro}

In formal models of \emph{machine learning} we have a concept class $C$ of
possible hypotheses, an unknown target concept $c^* \in C$
and training data given by correctly labelled random examples.
In formal models of {\em machine teaching} a collection of
labelled examples, a witness set $w$, is instead carefully chosen by a teacher $T$, i.e., $w=T(c^*)$, so the learner can reconstruct $c^*$, i.e., $c^*=L(w)$. In recent years, the field of machine teaching has
seen various applications in fields like 
pedagogy~\cite{SHAFTO201455}, 
trustworthy AI~\cite{zhu2018overview} and
explainable AI~\cite{yang2021mitigating,Hvardstun2023XAIWM}.


In reality, concepts are expressed in a representational language, and most languages have redundancy, with several representations mapping to the same concept; this is one common kind of {\em language bias}. For some languages, teacher and learner may not know whether two representations are equivalent ---it may not even be computable. Because of this, we may end up teaching more than one representation of a concept. 
For example, the `powers of 2', the `number of subsets of an $n$-element set', and the `number of $n$-bit binary strings', all denote the same subset of natural numbers, but teacher and learner may not know they fall in the same equivalence class. 



This situation is common in most languages with which humans and machines represent concepts, from natural languages to artificial neural networks. 
Consequently, studying how this relation between representations and concepts affects teaching has important implications not only for teaching, but also in machine learning in general (see, e.g., Hadley et al. \cite{hadley1999language}), 
where the notions of syntactic, semantic and search bias (any language bias) are explored to make a language more suitable for learning than another (see e.g., Muggleton and De Raedt  \cite{muggleton1994inductive}, Whigham \cite{whigham1996search}, Nédellec et al. \cite{bergadanodeclarative}). 
The issue is important in any area of AI where there is a non-injective correspondence from concepts to representations, such as explainable AI. Several notions of redundancy, density or compactness in semantics, syntax or representations have been studied in the literature \cite{yu1988can,enflo1994sparse,alpuente2010compact}, 
but they do not focus on quantifying redundancy or concentration over the equivalence classes, as we define them in this paper.  

We explore the idea of teaching {\em representations}, where the teacher has a representation in mind and wants the learner to identify it. 
Different representations $r_1$ and $r_2$ can be in the same equivalence class, a concept $c$, denoted by $[r_1] = [r_2] = c$. We model this much as in the traditional {\em concept} teaching scenario, but now by a set of representations $R$ expressed in a language, 
and the teacher $T: R \rightarrow W$ being a partial injective mapping of representations to witnesses. We do not require the mapping to be total to accommodate, e.g., the case $|R| > |W|$. As in the classical concept teaching setting, we assume the learner is able to check the consistency between a representation and  a given witness.

We assume the learner has simplicity functions on the set of representations and on witnesses. A similar assumption appears in various models of machine teaching, e.g., the Preference-Based teaching model assumes an ordering on concepts and the Recursive teaching model assumes a partial ordering on witnesses. Furthermore, in the teacher/learner algorithm of~\cite{telle2019teaching}, which we call Eager, the learner employs an ordering on concepts called the learning prior while the teacher wants to minimize the `teaching size', which implicitly sets an order on witnesses. 
In the Eager algorithm, the learner, when given a witness, will employ Occam's razor and guess the simplest consistent representation, and the teacher will always provide the smallest witness that suffices to learn the target representation. Note this implies the learner will learn at most one representation per concept, namely the simplest one. But the main issue is that if the teacher does not know that $r_1$ and $r_2$ are equivalent, and $r_1$ is simpler than $r_2$, then the teacher will be looking for a witness set for $r_2$ that distinguishes it from $r_1$, something the teacher will never find. 

How to alter the Eager algorithm in a natural way so that all the representations of a concept will be taught, not only the simplest one? 
When teaching a target it is quite common to give the smallest observation that suffices to learn it, and the learner could actually expect this from the teacher, and thus rule out representations that could have been taught with a witness smaller than what was given. Herein lies the simple idea behind the slight change to the Eager algorithm into what we call the Greedy algorithm. The teacher uses the smallest consistent witness to teach the simplest representation, removes this witness and representation from further consideration, and repeats. 
Likewise, the learner, when given a witness, will know the teacher behaves in the above way and be able to guess the correct representation.

We also introduce another protocol that we call Optimal, as it represents the graph-theoretical best possible under the minimal constraint that a representation must be taught by a consistent witness and that two distinct representations must be taught by two distinct witnesses. In Section~\ref{sec-comp} we show several formal results about the strengths of the three algorithms: Eager, Greedy and Optimal. In particular, on concepts taught by both Eager and Greedy, the latter will never use a larger witness, and often a smaller witness. However, Theorem~\ref{thm-concent} shows that for any concept class, adding redundant representations with small spread in the order, can level out this difference between Eager and Greedy. Moreover, there are situations where the max size witness used by Greedy will be much larger than the Optimal. Finally, in arguing for a lower bound on the performance of Greedy we arrive at a question about which binary matrix on $k$ distinct rows achieves the smallest number of 'projection vectors' (see Section \ref{sec-comp} for details). This question seems to be of combinatorial interest, and we conjecture that its solution is given by the matrix containing the binary representations of the numbers from zero to $k-1$.

In Section~\ref{sec-exper} we report on experiments with these three algorithms, on several representation languages with various kinds of redundancy.
%
Tables~\ref{table-small-results-overview_graph}, \ref{table-small-results-domains} and \ref{table-small-results-greedy_vs_eager} in the last section summarize the results. 
The main finding is that 
while Greedy needs more witness sets to teach all representations, it uses them more effectively to also teach more concepts, and it does so while teaching several common concepts by a smaller witness.
The experiments were furthermore chosen to answer the following question: what characterizes a language where Greedy teaches only a few common concepts by a smaller witness than Eager? We show that this is not directly related to the {\em amount} of redundancy, but rather to the {\em spread} of the redundant representations in the order, see Figure \ref{fig-preformance_vs_red_spread},
with Eager performing better the lower the Redundancy Spread. In Figure \ref{fig:2scat-bits-greedy}  we see that for P3 programs the witness sets are usually smaller than the programs they identify, which is an illuminating justification of why machine teaching from examples makes sense at all.

Teaching representations rather than concepts is a natural setting as teachers represent concepts in a given language, and size preferences (e.g., Occam's razor) apply to all representations and not only to the canonical one in the equivalence class, which teacher and learner may not know. The Eager and Greedy protocols are two extremes (one teaching a single representation per concept and the other teaching all representations) in a spectrum of approaches that can shed light to important phenomena such as the effect of redundancy in inductive search, the actual bias in size-related priors and the relevance of syntax over semantics in explanations. 

There  are many practical situations where representations rather than concepts matters. For instance, the Promptbreeder from GoogleMind \cite{DBLP:journals/corr/abs-2309-16797} that generates various prompts (representations) to get the same answer (concept), or the use of Codex in teaching \cite{DBLP:conf/ace/Finnie-AnsleyDB22} which produces various programs (representations) having the same behaviour (concept). 
Having said that, our main contribution is not practical, it is rather the insight about highlighting the distinction between concepts and representations, and showing theoretically and experimentally how small witness sets can be when teaching, given the existing redundancy. 
We conclude with a discussion in Section \ref{sec-disc}.


%
%
%
\section{Machine Teaching definitions}
\label{sec-defs}

Various models of machine teaching have been proposed, {e.g.,} the classical teaching dimension model~\cite{goldman1995complexity}, the optimal teacher model~\cite{balbach2008measuring}, recursive teaching~\cite{zilles2011models},  preference-based teaching~\cite{gao2017preference}, or no-clash teaching~\cite{no-clash}.
These models differ mainly in the restrictions they impose on the
learner and the teacher to avoid collusion.
The common goal is to keep the {\em teaching dimension}, i.e., the size of the largest teaching
set, $\max_{c \in C}|T(c)|$, as small as possible.

In all these models, we have a set $C$ of concepts, with each concept being a subset of a domain $X$. In this paper we consider 
more than one {\em representation} for the same concept, thus with $R$ being a set of representations that constitute a multiset of the concepts $C$ that are the subsets of $X$. 
The teacher $T : R \rightarrow W$ is now a partial injective mapping from representations $r \in R$ to a set of observations $w \in W$ (often $W \subseteq 2^{X \times \{0,1\}}$ so that $w$ is a set of negatively or positively labelled examples from $X$) and the learner $L: W \rightarrow R$ is a  partial mapping in the opposite direction. 
As usual, teacher and learner share the consistency graph $G_R$ on vertex set $R \cup W$ with a representation $r \in R$ adjacent to a witness $w \in W$ (and thus $rw \in E(G_R)$ an edge) if $r$ and $w$ are consistent, i.e., with positive (resp. negative) examples in $w$ being members (resp. non-members) of $r$. 
Naturally, $r$ must be consistent with  $T(r)$ and $L(w)$ must be consistent with $w$, and a successful teacher-learner pair must have $L(T(r_1))=r_2$ such that $[r_1]=[r_2]$. In what follows, we will assume $L$ and $T$ use the same representation language and we will impose that $L(T(r))=r$.


Let us take the graph theoretical view. The teacher and learner mappings are matchings in $G_R$ between representations and witnesses. 
Two vertices are called twins if they have the same set of neighbors.
Usually there is a bijection between the equivalence classes of twins of $R$ in the graph $G_R$, that we denote by $R/W$, and the set of concepts $C$. However, if the witness set $W$ is too sparse then $R/W$ may be a coarsening of $C$.

We will be comparing two models for teaching representations, that we call Eager and Greedy. Both assume some natural size functions on representations $R$ and on witnesses $W$, and use these to arrive at two total orderings $\precdot_R$ on $R$ and $\precdot_W$ on $W$.
For example, if the representations in $R$ are expressed in some description language, which
can be English, or a programming language, or some set of mathematical formulas, then a representation consists of finite strings of symbols drawn from some fixed alphabet $\Sigma$, given  with a total order which will be used to derive a lexicographic order on strings over this alphabet, with shorter strings always smaller, sometimes called shortlex. 
To define a total ordering on witnesses, we can do something similar.

Thus, the input to these algorithms is what we call an ordered consistency graph consisting of the 3-tuple $(G_R,\precdot_R, \precdot_W)$ and when we talk about a vertex (or representation/witness) appearing earlier than another we mean in these orderings. In the Eager model we have $L(w)=r$ for $r$ being the earliest representation (in the order  $\precdot_R$) consistent with $w$. The teacher, knowing that the learner behaves in this way, will construct the mapping $T:R \rightarrow W$ iteratively as follows: go through $W$ in the order of  $\precdot_W$, and for a given witness $w$ find the earliest $r \in R$ with $rw \in E(G_R)$, and if $T(r)$ not yet defined then set $T(r)=w$ else continue with next witness.
In~\cite{telle2019teaching} the Eager model was used to teach programs in the Turing-complete language P3, by witnesses containing specified I/O-pairs. Many P3 programs have equivalent I/O-specifications, and we can observe that in the Eager model only the earliest representation (program) of any concept (I/O-specification) is taught.  

In order to teach all representations we introduce the Greedy model, where we make a slight change to the way the teacher constructs its mapping, as follows: go through $W$ in the order of  $\precdot_W$, and for a given witness $w$ find the earliest $r \in R$ with $rw \in E(G_R)$ such that $T(r)$ is not yet defined, then set $T(r)=w$ and continue with next witness (if no such $r$ exists then drop this $w$). Teaching and Learning in the Greedy model is thus done following an order, similar to what happens in the Recursive teaching model of~\cite{zilles2011models} (teach sequentially and remove concepts that have already been taught), using also an order on the representations, similar to what is done in the Preference-Based teaching model of~\cite{gao2017preference} (return the most preferred remaining concept consistent with a given witness).


In classical machine teaching, the goal is to minimize the teaching dimension of a concept class, i.e. to find a legal teacher mapping $T:C \rightarrow W$ where the maximum of $|T(c)|$ over all $c \in C$ is minimized.
When there is a total mapping $T: R \rightarrow W$ we do the same for teaching representations, except we use $size(T(c))$ rather than cardinality, as in \opta defined below. However, a total mapping may not exist, and in that case we will in this paper define an alternative measure, as in \optb below, being the maximum number of representations covered when constrained to partial teacher mappings covering the maximum number of concepts.
We thus define the Optimal model in two flavors.

\opta: this is the minimum max size witness used over all teacher mappings, under the sole restriction that the teacher mapping is injective and assigns a consistent witness to every representation. This value can be computed in polynomial time by a binary search for the smallest $k$ such that the induced subgraph $G_R^k$ of $G_R$, on $R$ and all witnesses of size at most $k$, has a matching saturating $R$, i.e., containing all of $R$. 

\optb: as \opta is undefined when $G_R$ has no matching saturating $R$, 
we define \optb as the maximum number of representations covered by any matching that maximizes the number of concepts covered.
It is not clear that this number is computable in polynomial time, but for most graphs 
in the experimental section we could compute it.

\section{Formal comparisons}
\label{sec-comp}

We show several results comparing the performance of the three algorithms Eager, Greedy and Optimal.
We start with Figure~\ref{fig-example_of_eager_vs_greedy_vs_optimal} giving the simplest possible consistency graph where Eager, Greedy and \opta behave differently.

The teacher mapping $T: R \rightarrow W$ for Greedy, as defined in previous section, is computed by an iterative procedure whose outer loop follows the witness ordering (and inner loop representation ordering).  Let us call it $T_W: R \rightarrow W$.
Consider the alternative $T_R: R \rightarrow W$ which switches these two:
go through $R$ in the order of $\precdot_R$, and for a given representation $r$ find the earliest $w \in W$ with $rw \in E(G_R)$ such that $T_R^{-1}(w)$ is not yet defined, then set $T_R(r)=w$ and continue with the next representation (if no such $w$ exists then drop this $r$). 

\begin{theorem}
The teacher mappings returned by Greedy following $\precdot_W$ versus the alternative following $\precdot_R$ are the same.
\end{theorem}

\begin{proof}
We prove, by induction on $\precdot_W$, the statement (*): "for any $w \in W$ if $T_W(r)=w$ then also $T_R(r)=w$".
Let $w_1$ be the earliest witness. We have $T_W(r')=w_1$ and when assigning $T_R$ clearly $w_1$ is the earliest neighbor of $r'$ with $T_R^{-1}(w_1)$ undefined so that $T_R(r')=w_1$. For the inductive step, we assume the statement (*) for all witnesses earlier than $w$, and assume that  $T_W(r)=w$.  By the inductive assumption (*) we know that no witness earlier than $w$ will by $T_R$ be assigned to $r$. This means that when assigning $T_R(r)$ then $w$ is the earliest neighbor of $r$ with $T_R^{-1}(w)$ undefined, and thus we have $T_R(r)=w$ as desired.
\end{proof}


Since the Greedy teaching is only a slight change to the Eager teaching, it is easy to see that for any $r \in R$, if Eager assigns $T_E(r)=w_E$ then Greedy will also assign some $T_G(r)=w_G$, and in the order $\precdot_W$ we could have $w_G$ earlier than $w_E$, but not the other way around. 

\begin{theorem}
On representations/concepts taught by Eager, Greedy will never use a larger witness.
\end{theorem}

\begin{proof}
This since $T_E(r)=w_E$ implies that $r$ is the earliest representation consistent with $w_E$, so when the Teacher mapping of Greedy is computed then when reaching $w_E$ the first representation tried will be $r$, and if $T_G(r)$ has already been defined it will be for an earlier and thus no larger witness. 
\end{proof}


 




We show that for any concept class we can add redundant copies of representations consecutive in $\precdot_R$ to make Eager and Greedy perform almost identical.

\begin{theorem}
\label{thm-concent}
    For any concept class $R$ (i.e. each concept has a unique representation) on witness set $W$ and orders $\precdot_R,\precdot_W$, we can add less than a total of $|W|$ copies of representations to get $R'$, and make $\precdot_{R'}$ an extension of  $\precdot_R$ with all copies consecutive, such that Eager teaches the same on $(G_R, \precdot_R, \precdot_W)$ as on $(G_{R'},\precdot_{R'},\precdot_W)$ and using the same witnesses as Greedy on $(G_{R'},\precdot_{R'},\precdot_W)$ on all common concepts.
\end{theorem}

\begin{proof}
Consider the ordered consistency graph $(G_R, \precdot_R,\precdot_W)$ and for each representation $r \in R$ compute $f_r=|\{w \in W: rw \in E(G_R) \wedge (r'w \in E(G_R) \Rightarrow r \precdot_R r')\}|$, i.e. the number of witnesses for which $r$ is the earliest neighbor. Let $R'$ be the representation class corresponding to the ordered consistency graph $(G_{R'}, \precdot_{R'},\precdot_W)$ where for each $r \in R$ we add $f_r-1$ copies of $r$  and extend the order $\precdot_R$ to $\precdot_{R'}$ by putting all the copies consecutively right after $r$, so there are $f_r$ consecutive copies total of $r$.
Note that, as the sum of $f_r$ over all $r \in R$ is $|W|$ (each witness has a single earliest $r$) we add less than $|W|$ new representations.

For each $w \in W$ define $r_w$ to be the earliest neighbor of $w$ in $G_R$ and let $r'_w$ be the earliest neighbor in $G_{R'}$. For simplicity, when we say $G_R,G_{R'}$ we mean the ordered consistency graphs. We prove the following loop invariant on any witness $w \in W$ by induction over $\precdot_W$: "Knowing $r_w$ has $f_{r_w}$ witnesses for which it is the earliest, let $w$ be the $k$th such witness. Both on $G_R$ and $G_{R'}$, if $k=1$ then Eager will assign $r_w$ to $w$, and if $k>1$ Eager will not use $w$. Greedy on $G_{R'}$ will assign $r$ to $w$ for $r$ being the $k$th copy of $r_w$ in the $\precdot_{R'}$ order."
The loop invariant is trivially true for the first witness. For the inductive step, we know $r_w$ is the earliest neighbor of $w$ and that $w$ shares this property with $f_{r_w}$ witnesses. If $w$ is the first of these witnesses, then it is clear that both Eager and Greedy will assign $r_w$ to $w$, as the loop invariant tells us no earlier witness has been assigned to $r_w$. If $w$ is the $k$th copy of $r_w$ for $k \geq 2$ then Eager will not use $w$, neither on $G_R$ nor $G_{R'}$, while for Greedy on $G_{R'}$ as there are at least $k$ copies of $r_w$ all in consecutive order then Greedy will assign $w$ to the $k$th copy as it must be available by the loop invariant.
We have thus proven the loop invariant, and this implies the statement of the Theorem as it encompasses all representations taught by both algorithms.
\end{proof}

%

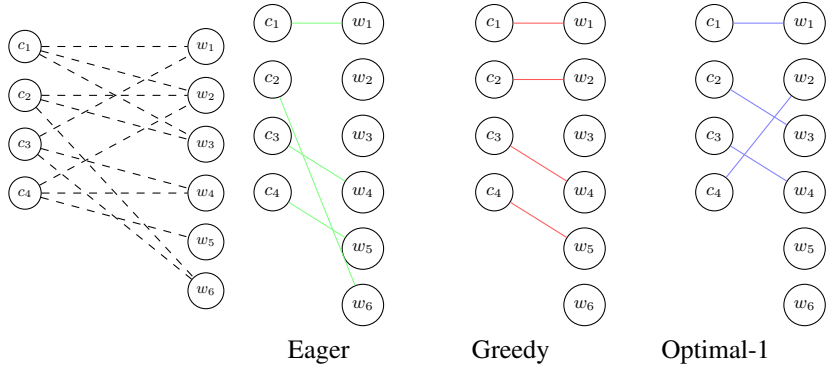
\begin{figure}[t]
\begin{center}
\begin{tabular}{@{}c}
\begin{tikzpicture}[scale=0.65, transform shape]  
    \def\globalSpaceC{1}
    \def\globalSpace{1.0}
    \node[circle, draw, minimum size=0.2cm, align=center] (c_1) at (-3.5,0){\small $c_1$};
    \node[circle, draw, minimum size=0.2cm, align=center] (c_2) at (-3.5,-\globalSpaceC){\small $c_2$};
    \node[circle, draw, minimum size=0.2cm, align=center] (c_3) at (-3.5,-\globalSpaceC*2){\small $c_3$};
    \node[circle, draw, minimum size=0.2cm, align=center] (c_4) at (-3.5,-\globalSpaceC*3){\small $c_4$};

    \def\hS{0.33} 

    \def\currY{0}
    \node[circle, draw, minimum size=0.2cm] (w_{x1}) at (0.2,\currY) {\small $w_1$};
    \pgfmathsetmacro{\currY}{\currY - \globalSpace} 
    \node[circle, draw, minimum size=0.2cm] (w_{x2}) at (0.2,\currY) {\small $w_2$};
    \pgfmathsetmacro{\currY}{\currY - \globalSpace} 
    \node[circle, draw, minimum size=0.2cm] (w_{x3}) at (0.2,\currY) {\small $w_3$};
    \pgfmathsetmacro{\currY}{\currY - \globalSpace} 
    \node[circle, draw, minimum size=0.2cm] (w_{x4}) at (0.2,\currY) {\small $w_4$};
    \pgfmathsetmacro{\currY}{\currY - \globalSpace} 
    \node[circle, draw, minimum size=0.2cm] (w_{x5}) at (0.2,\currY) {\small $w_5$};
    \pgfmathsetmacro{\currY}{\currY - \globalSpace} 
    \node[circle, draw, minimum size=0.2cm] (w_{x6}) at (0.2,\currY) {\small $w_6$};

    \draw[color=black, dashed] (c_1) -- (w_{x1});
    \draw[color=black, dashed] (c_1) -- (w_{x2});
    \draw[color=black, dashed] (c_1) -- (w_{x3});
    \draw[color=black, dashed] (c_2) -- (w_{x2});
    \draw[color=black, dashed] (c_2) -- (w_{x3});
    \draw[color=black, dashed] (c_2) -- (w_{x6});
    \draw[color=black, dashed] (c_3) -- (w_{x1});
    \draw[color=black, dashed] (c_3) -- (w_{x4});
    \draw[color=black, dashed] (c_3) -- (w_{x6});
    \draw[color=black, dashed] (c_4) -- (w_{x2});
    \draw[color=black, dashed] (c_4) -- (w_{x4});
    \draw[color=black, dashed] (c_4) -- (w_{x5});

\end{tikzpicture}
\vspace{0.3cm}  
\end{tabular}
\begin{tabular}{@{}cc@{}cc@{}}
\begin{tikzpicture}[scale=0.75, transform shape]  
    \def\globalSpaceC{1}
    \def\globalSpace{1.0}
    \node[circle, draw, minimum size=0.2cm, align=center] (c_1) at (-1.4,0){\small $c_1$};
    \node[circle, draw, minimum size=0.2cm, align=center] (c_2) at (-1.4,-\globalSpaceC){\small $c_2$};
    \node[circle, draw, minimum size=0.2cm, align=center] (c_3) at (-1.4,-\globalSpaceC*2){\small $c_3$};
    \node[circle, draw, minimum size=0.2cm, align=center] (c_4) at (-1.4,-\globalSpaceC*3){\small $c_4$};

    \def\hS{0.33} 

    \def\currY{0}
    \node[circle, draw, minimum size=0.2cm] (w_{x1}) at (0.2,\currY) {\small $w_1$};
    \pgfmathsetmacro{\currY}{\currY - \globalSpace} 
    \node[circle, draw, minimum size=0.2cm] (w_{x2}) at (0.2,\currY) {\small $w_2$};
    \pgfmathsetmacro{\currY}{\currY - \globalSpace} 
    \node[circle, draw, minimum size=0.2cm] (w_{x3}) at (0.2,\currY) {\small $w_3$};
    \pgfmathsetmacro{\currY}{\currY - \globalSpace} 
    \node[circle, draw, minimum size=0.2cm] (w_{x4}) at (0.2,\currY) {\small $w_4$};
    \pgfmathsetmacro{\currY}{\currY - \globalSpace} 
    \node[circle, draw, minimum size=0.2cm] (w_{x5}) at (0.2,\currY) {\small $w_5$};
    \pgfmathsetmacro{\currY}{\currY - \globalSpace} 
    \node[circle, draw, minimum size=0.2cm] (w_{x6}) at (0.2,\currY) {\small $w_6$};

    \draw[color=green] (c_1) -- (w_{x1});
    \draw[color=green] (c_2) -- (w_{x6});
    \draw[color=green] (c_3) -- (w_{x4});
    \draw[color=green] (c_4) -- (w_{x5});

\end{tikzpicture}
&
\hspace{0.7cm}
\begin{tikzpicture}[scale=0.75, transform shape]  
    \def\globalSpaceC{1}
    \def\globalSpace{1.0}
    \node[circle, draw, minimum size=0.2cm, align=center] (c_1) at (-1.4,0){\small $c_1$};
    \node[circle, draw, minimum size=0.2cm, align=center] (c_2) at (-1.4,-\globalSpaceC){\small $c_2$};
    \node[circle, draw, minimum size=0.2cm, align=center] (c_3) at (-1.4,-\globalSpaceC*2){\small $c_3$};
    \node[circle, draw, minimum size=0.2cm, align=center] (c_4) at (-1.4,-\globalSpaceC*3){\small $c_4$};

    \def\hS{0.33} 

    \def\currY{0}
    \node[circle, draw, minimum size=0.2cm] (w_{x1}) at (0.2,\currY) {\small $w_1$};
    \pgfmathsetmacro{\currY}{\currY - \globalSpace} 
    \node[circle, draw, minimum size=0.2cm] (w_{x2}) at (0.2,\currY) {\small $w_2$};
    \pgfmathsetmacro{\currY}{\currY - \globalSpace} 
    \node[circle, draw, minimum size=0.2cm] (w_{x3}) at (0.2,\currY) {\small $w_3$};
    \pgfmathsetmacro{\currY}{\currY - \globalSpace} 
    \node[circle, draw, minimum size=0.2cm] (w_{x4}) at (0.2,\currY) {\small $w_4$};
    \pgfmathsetmacro{\currY}{\currY - \globalSpace} 
    \node[circle, draw, minimum size=0.2cm] (w_{x5}) at (0.2,\currY) {\small $w_5$};
    \pgfmathsetmacro{\currY}{\currY - \globalSpace} 
    \node[circle, draw, minimum size=0.2cm] (w_{x6}) at (0.2,\currY) {\small $w_6$};

    \draw[color=red] (c_1) -- (w_{x1});
    \draw[color=red] (c_2) -- (w_{x2});
    \draw[color=red] (c_3) -- (w_{x4});
    \draw[color=red] (c_4) -- (w_{x5});

\end{tikzpicture}
&
\hspace{1.1cm}
\begin{tikzpicture}[scale=0.75, transform shape]  
    \def\globalSpaceC{1}
    \def\globalSpace{1.0}
    \node[circle, draw, minimum size=0.2cm, align=center] (c_1) at (-1.4,0){\small $c_1$};
    \node[circle, draw, minimum size=0.2cm, align=center] (c_2) at (-1.4,-\globalSpaceC){\small $c_2$};
    \node[circle, draw, minimum size=0.2cm, align=center] (c_3) at (-1.4,-\globalSpaceC*2){\small $c_3$};
    \node[circle, draw, minimum size=0.2cm, align=center] (c_4) at (-1.4,-\globalSpaceC*3){\small $c_4$};

    \def\hS{0.33} 

    \def\currY{0}
    \node[circle, draw, minimum size=0.2cm] (w_{x1}) at (0.2,\currY) {\small $w_1$};
    \pgfmathsetmacro{\currY}{\currY - \globalSpace} 
    \node[circle, draw, minimum size=0.2cm] (w_{x2}) at (0.2,\currY) {\small $w_2$};
    \pgfmathsetmacro{\currY}{\currY - \globalSpace} 
    \node[circle, draw, minimum size=0.2cm] (w_{x3}) at (0.2,\currY) {\small $w_3$};
    \pgfmathsetmacro{\currY}{\currY - \globalSpace} 
    \node[circle, draw, minimum size=0.2cm] (w_{x4}) at (0.2,\currY) {\small $w_4$};
    \pgfmathsetmacro{\currY}{\currY - \globalSpace} 
    \node[circle, draw, minimum size=0.2cm] (w_{x5}) at (0.2,\currY) {\small $w_5$};
    \pgfmathsetmacro{\currY}{\currY - \globalSpace} 
    \node[circle, draw, minimum size=0.2cm] (w_{x6}) at (0.2,\currY) {\small $w_6$};

    \draw[color=blue] (c_1) -- (w_{x1});
    \draw[color=blue] (c_2) -- (w_{x3});
    \draw[color=blue] (c_3) -- (w_{x4});
    \draw[color=blue] (c_4) -- (w_{x2});

\end{tikzpicture}
\\ Eager 
& Greedy 
& Optimal-1 
\end{tabular}
\end{center}
  \vspace{-0.05in}
  \caption{Simplest possible ordered consistency graph ($\precdot_R$ and $\precdot_W$ given by the indices) 
  showing that Eager, 
   Greedy 
  and Optimal-1 
  can require different maximum teaching sizes (witnesses of size 6, 5 and 4 respectively), assuming witness $w_i$ has size $i$.}
  \label{fig-example_of_eager_vs_greedy_vs_optimal}
  \vspace{-0.05in}
\end{figure}


When 
we allow the size of ground elements to vary, we can show the following.

\begin{theorem}
    When there is a bound on the number of witnesses of any size, then for any $t$, there exists a size ordered consistency graph where Greedy uses a witness of size $G_{max}$ while \opta will not use a witness of size larger than $O_{max}$ and where $G_{max} - O_{max} \geq t$.
\end{theorem}

\begin{proof}
    We prove this statement by construction, for a graph where we have only one representation for each concept. Assume there are at most $s$ witnesses of any size. 
    Let concept ordering be $c_1,c_2,...c_k$, and let  witness ordering be $w_1, w_2, w_3, ...w_{st+k}, ...$ with $s(w_i) \leq s(w_{i+1})$, for any choice of $k$.
    Let $c_1$ be consistent with $w_1$ and $w_2$, and let $c_2$ be consistent with $w_1$ and $w_{st+k}$ (but not consistent with any witness in between, which could happen if e.g. there are that many ground elements) and let $c_i$ for $3 \leq i \leq k$ be consistent with $w_i,w_{i+1},...,w_{i+st}$.
The Optimal matching will assign $c_1$ to $w_2$ and $c_2$ to $w_1$ and $c_i$ to $w_i$ for any $3 \leq i \leq k$, thus using maximum size witness with index $k$. Greedy will assign almost the same matching except it starts with $c_1$ to $w_1$ and then $c_2$ to $w_{st+k}$ since this is the only one available, thus using maximum size witness with index $st+k$. As we have at most $s$ witnesses of each size, we have that $s(w_{st+k}) - s(w_k) \geq t$.
\end{proof}

Finally, we would like to give lower bounds on the performance of Greedy, in the setting where size of witnesses is simply cardinality, i.e. the traditional teaching dimension. 
Let us consider two extremes, one where we have many representations of the same single concept (i.e. $|C|=1$) and another where each concept has a single representation (i.e. $R=C$). 

In the former case, it is easy to see that Greedy will perform as well as \opta, since Greedy goes through witnesses by increasing cardinality and assigns them to the next representation (all of the same single concept) if consistent. Thus if Greedy assigns a witness of size $q$ then so must \opta.

In the latter case, with only one representation per concept, we  argue that if Greedy assigns a witness of cardinality $q$ to some concept $c$ then this implies a lower bound on $|C|$ given as follows. 
Let us start by asking, why was $c$ not taught by a smaller witness?  Assuming there are $|X|=n$ examples, then any subset $Q \subseteq X$ of size $q-1$ when labelled consistent with $r$ has already been tried by Greedy, and hence some other concept must already have been assigned to any such $Q$, and all these concepts are distinct. This means we must have taught ${n \choose q-1}=k$ other concepts already. But then we have already taught at least $k+1$ concepts and we can again ask why were any of these not taught by a smaller witness of size $q-2$? It must be that any such witnesses (labelled to be consistent with some concept among the k+1 we already have) must have been used to teach other, again distinct, concepts. 

Consider the binary matrix $M$ where concepts are rows and columns are given by $X$ and $M(c,x)=1$ if and only if $c$ consistent with $x$ labelled 1.
Note that, to verify how many distinct witnesses
exists, corresponding to new concepts, that are labelled consistently with one of these $k + 1$ concepts, one must
sum up the number of distinct rows when projecting on $q-2$ columns, for all
choices of these columns. Note that the number of distinct rows, i.e witnesses and hence number of concepts, when projecting on $q-2$ columns,
for all choices of these columns, depends on the matrix $M$ we do the projection on. 
Thus, we need
to find the matrix $M$ minimizing the sum of unique rows after doing the projection.
To achieve a lower bound on the size of $C$ when Greedy uses a witness of size $q$, we thus arrive at the following combinatorial question. What is the binary matrix $M$ on $k$ distinct rows and $n$ columns that would give the smallest sum when projecting on $q$ columns? 

We have not been able to answer this question, and until we answer it we are not able to carry out the final steps of the lower bound argument.
However, we conjecture that it is achieved by the matrix consisting of the $k$ rows corresponding to the binary representations of the numbers between zero and $k-1$, with leading 0s to give them length $n$.


\section{Experiments }
\label{sec-exper}

In this section we  observe quantitatively how each teaching protocol behaves in a variety of languages with various types of redundancy.
Table~\ref{table-small-results-overview_graph} summarizes the key features. 
Let us first describe the metrics used, before describing the languages. 
To estimate how redundant a language is, we define \emph{Redundancy} as $1-Uniqueness$, where $Uniqueness$ is computed as the average for all concepts of $1/{|R_{c}|}$, with ${R_{c}}$ the set of all equivalent representations for concept $c$. 
Formally:
\begin{equation}
   \mathrm{Redundancy} = 1-{\frac{1}{|C|}\sum_{c \in C} {\frac{1}{|R_{c}|}}}~. 
   \label{eq-redundancy}
\end{equation}

From the result of Theorem~\ref{thm-concent} we suspect that Eager performs well compared to Greedy when, for each witness $w$, if we look at the set $R^{i}_w$ of the $i$ earliest representations consistent with $w$, then there are few different concepts in $R^{i}_w$.  
Assuming the $\precdot_W$ ordering is $w_1,w_2,...$ we define \emph{Redundancy Spread} as the average number of different concepts in $R^i_{w_i}$. We use 
$R^i_{w_i}$ for $w_i$ because $\leq i$ values of $T:R \rightarrow W$ have been defined when Greedy or Eager consider $w_i$. 
Formally, if for each witness $w_i$ we let $r^j_{w_i}$ be the $j$-th representation consistent with $w_i$, and recalling that $[r]$ is the concept equivalence class $r$ belongs to, the definition of Redundancy Spread becomes:
\begin{align}
 \label{eq-spread}  
 \frac{1}{|W|} \sum_{w_i\in W} \left | \{ [r^{j}_{w_i}] : j \in [1 \ldots i] \right\}| \nonumber
\end{align}


\subsection{Boolean Expressions}
\label{sec-exper-setup-boolean}

We experiment with Boolean expressions, since in this domain it is possible to decide whether two representations belong to the same equivalence class, i.e., correspond to the same concept.  Furthermore, in this language, we are able to measure the degree of redundancy present in the set of representations. In all our experiments, we consider the ordered alphabet of 3 Boolean variables $\{a,b,c\}$ and every representation in $R$ is expressed in DNF (Disjunctive Normal Form), i.e. OR of AND-terms. Every variable occurs at most once in a term, possibly preceded by the negation operator $\neg$.
For example, $(a \wedge b \wedge \neg c)$ is a valid ground term.

Let us first describe the witnesses, which are similar across all experiments. 
Each example is a pair consisting of a binary string of length three (the input) and a bit (the output), such as $(001,1)$ indicating that when $a=0,b=0,c=1$ the value of the function should be $1$. There are thus 16 witnesses of cardinality one, and $16*14/2=112$ witnesses of cardinality two, as we do not allow for self-incongruent witnesses and the order of examples is irrelevant. For these experiments we use witnesses of cardinality up to $5$, and only positive examples, which implies $|W|=3,448$.  

$\precdot_W$: The witnesses in the collection $W$ are ordered by their increasing teaching size. The size function $s : W \rightarrow \mathbb{N}$ is defined for a witness $w$ as the number of examples in $w$ times 4 (the number of bits) plus the number of 1s in the inputs of the examples. This scheme would be useful for a situation with many variables, where it could be more efficient to just identify the position of the 1s in a long binary string. 


In the first experiment, \emph{3-DNF}, we build a scenario without redundancy. We consider all 256 Boolean functions on three variables as the set $R = C$, as follows. A Boolean function $\phi$ on $a,b,c$ is uniquely defined by a truth table that gives the truth values for the 8 possible truth assignments to the three variables. For any assignment  $(v_a,v_b,v_c)$ to $(a,b,c)$ where $\phi(v_a,v_b,v_c)=1$ we include the corresponding term in the representation $r_{\phi}$ of $\phi$ in $R$. For instance, if $\phi$ is logically equivalent to ``$a$ and $b$'' then this is represented as $r_{\phi}=(a\wedge b\wedge c) \vee (a\wedge b\wedge \neg c)$.


In our second experiment, \emph{3-term DNF}, we introduce redundancy in $R$, and do this in a way that will shrink the concept space to have smaller cardinality $|C|=246$. Every representation is now a disjunction of up to $3$ ground terms, and each term contains one, two or three distinct variables, each possibly negated. For a single term we have six representations of one variable, $12=6*4/2$ representations of two variables, and $8$ representations for three variables, thus $26$ in total. For two terms we have $26*25/2=325$ representations, and for three terms we have $26*25*24/3!=2,600$.
We also have a single representation (False) with no terms, and thus $|R|=1+26+325+2,600=2,952$. 


For the third experiment, \emph{3-term DNF with permutations}, we want to increase Redundancy 
and decrease Redundancy Spread. The set of representations is similar to 3-term DNF, but here we allow for any permutation of variables within each term.
For example, if $(b \wedge \neg c) \in R$, also $(\neg c\wedge b) \in R$. This drastically increases the cardinality of $R$, while Redundancy also increases substantially and Redundancy Spread decreases.
This third experiment is carried out also against a smaller witness set, with witnesses of cardinality 5 only, to see the effects of a large decrease in Redundancy Spread.

Finally, the last Boolean experiment, \emph{3-term DNF with permutations and duplicates}, is similar to the previous one, 3-term DNF with permutations, but here each representation occurs consecutively twice in the ordered collection $R$, for yet another variation of redundancy.

$\precdot_R$: Representations in $R$ are ordered by increasing size $s(r)$ defined as the number of literals in $r$ plus the number of negated variables plus the number of disjunction symbols; e.g.,
$s((a\wedge b\wedge c) \vee (a\wedge b\wedge \neg c))=8$. 
Two representations of same size are further ordered by the increasing number of terms, and by the $\langle a, b, c, \neg, \vee \rangle$-induced lexicographic order.

\begin{table*}[]
    \centering
    \begin{tabular}{|l|r|r|c|r|c|c|}
    \hline
        \multicolumn{1}{|p{0.26\linewidth}|}{\centering \small \multirow{2}{6em}{Domain}}
        & \multicolumn{1}{|p{0.05\linewidth}|}{\centering \multirow{2}{1em}{$|R|$}}
        & \multicolumn{1}{|p{0.03\linewidth}|}{\centering \small \multirow{2}{1em}{$|C|$}}
        & \multicolumn{1}{|p{0.07\linewidth}|}{\centering \multirow{2}{4em}{Witness}}
        & \multicolumn{1}{|p{0.05\linewidth}|}{\centering \small \multirow{2}{1em}{$|W|$}}
        & \multicolumn{1}{|p{0.1\linewidth}|}{\centering \small \multirow{2}{6em}{Redundancy}}
        & \multicolumn{1}{|p{0.1\linewidth}|}{\centering \small \multirow{2}{6em}{  Redundancy \\  \centering Spread}}
        \\ 
         & & & & & & \\
        \hline
        3-DNF & 256 & 256 
        & {\centering $|w|\leq5$ } & 3,488 
        & {\centering 0} & {\centering 15.13} \\
        \hline
        3-Term DNF & 2,952 & 246 
        &  {\centering $|w|\leq5$ } & 3,488 
        & {\centering 0.727} & {\centering 13.28} \\
        \hline
        \multirow{2}{8em}{3-Term DNF with permutations} & \multirow{2}{2.8em}{79,158} & \multirow{2}{1.5em}{246}   
        & $|w| \leq 5$ & 3,488 
        & 0.9896 & 11.41 \\ 
        \hhline{|~|~|~|-|-|-|-|} 
        &  &  & $|w|=5$ & 1,792 
        & 0.9896 & 7.04\\
        \hline
         \multirow{2}{12em}{3-Term DNF with permuts. and duplicates} & \multirow{2}{3.3em}{158,316}   & \multirow{2}{1.5em}{246}   
         & \multirow{2}{*}{\small{$|w|\leq5$}} & \multirow{2}{2.2em}{3,488} & 
         \multirow{2}{4em}{\centering{0.9948}} & \multirow{2}{4em}{\centering{10.42}}\\
         & & & & &  & \\
        \hline
        P3 & $1.9*10^9$
        & {\centering{?}} & $s(w) \leq 6$   
        & 6,548 & ? 
        & ?\\ \hline
        small-P3 & 1,267
        & 106*  
        &$s(w) \leq 4$   
        & 260 
        & 0.331 & 2.97\\ \hline
    \end{tabular}
    \caption{
    Domain features: 
    `Witness' is the size of each witness.
    For P3 we do not know $G_R$ nor $C$.
    *For small-P3 we do not know $|C|$ so we show  $|R/W|$. }
    \label{table-small-results-overview_graph}

    \vspace{0.25in}
    \centering
    \begin{tabular}{|l|c|l|r|r|r|r|}
    \hline
        \multicolumn{1}{|p{0.22\linewidth}|}{\centering \small \multirow{2}{2em}{Domain}} 
        & \multicolumn{1}{|p{0.07\linewidth}|}{\centering \multirow{2}{3em}{Witness}}
        & \multicolumn{1}{|p{0.14\linewidth}|}{\centering \small \multirow{2}{5em}{Algorithm}}
        & \multicolumn{1}{|p{0.08\linewidth}|}{\centering \small Reps. taught}
        & \multicolumn{1}{|p{0.08\linewidth}|}{\centering \small Concepts taught}
        & \multicolumn{1}{|p{0.07\linewidth}|}{\centering \small Max. $\{s(w)\}$}
        & \multicolumn{1}{|p{0.08\linewidth}|}{\centering \small Max. $\{i(w)\}$}
        \\ 
        \hline
        \hhline{|=|=|=|=|=|=|=|}
        \hline
        \multirow{3}{4em}{3-DNF}&\multirow{3}{*}{\centering{  $|w|\leq5$}}
         & \eager & 219 & 219 & 30 & 3,488 \\  
         \hhline{|~|~|-|-|-|-|-|}  
        && \greedy & 256 & 256 & 16 & 328 \\            
        \hhline{|~|~|-|-|-|-|-|}
        && \opta & 256 & 256 & 16 & 256 \\
        \hhline{|=|=|=|=|=|=|=|}
        \multirow{3}{6em}{3-Term DNF} &\multirow{3}{*}{\centering{  $|w|\leq5$}}
        & \eager  & 170 & 170 & 30 & 3,466 \\
        \hhline{|~|~|-|-|-|-|-|}
        && \greedy & 2,895 & 246 & 30 & 3,481 \\
        \hhline{|~|~|-|-|-|-|-|}        
        && \opta & 2,952 & 246 & 28 & 2,952 \\
        \hhline{|=|=|=|=|=|=|=|} 

        \multirow{6}{10em}{3-Term DNF \\with permutations}
        & \multirow{3}{*}{\centering{$|w|\leq5$}}
        & \eager & 170 & 170 & 30 & 3,466 \\ \hhline{|~|~|-|-|-|-|-|} 
        &&\greedy & 3,488 & 189 & 30 & 3,488 \\ \hhline{|~|~|-|-|-|-|-|} 
         
        &&\optb & 3,488 & 246 & 30 & 3,488 \\ 
        \hhline{|~|=|=|=|=|=|=|}
        & \multirow{3}{*}{\centering{$|w|=5$}}
        & \eager& 170 & 170 & 30 & 1,770 \\ \hhline{|~|~|-|-|-|-|-|} 
        &&\greedy & 1,792 & 177 & 30 & 1,792 \\ \hhline{|~|~|-|-|-|-|-|} 
        &  &\optb &1,792 & 246 & 30 & 1,792 \\ 
        \hhline{|=|=|=|=|=|=|=|}
        \multirow{3}{10em}{3-Term DNF \\with permutations \\and duplicates}  
        &  \multirow{3}{*}{\centering{$|w|\leq5$}}
        & \eager & 170 & 170 & 30 & 3,466 \\
         \hhline{|~|~|-|-|-|-|-|} 
        &&\greedy  & 3,488 & 178 & 30 & 3,488 \\
         \hhline{|~|~|-|-|-|-|-|} 
        && \optb  & 3,488 & 246 & 30 & 3,488 \\
        \hhline{|=|=|=|=|=|=|=|} 
        \multirow{2}{4em}{P3} & \multirow{2}{*}{\centering{ $s(w)\leq6$}}
        &\eager & 2,032 & 2,032 & 6 & 6,512 \\ \hhline{|~|~|-|-|-|-|-|}
        && \greedy  & 6,548 &  & 6 & 6,548 \\ 
        \hhline{|=|=|=|=|=|=|=|} 
        \multirow{3}{4em}{small-P3} & \multirow{3}{*}{\centering{  $s(w)\leq4$}}
        &\eager & 53 & 53 & 4 & 202 \\ \hhline{|~|~|-|-|-|-|-|}
        && \greedy  & 225 & 65 & 4 & 260 \\ \hhline{|~|~|-|-|-|-|-|}
        && \optb & $\geq 214$* & 106 & 4 & 260 \\  
        \hline
    \end{tabular}
     \caption{Overview of results for the teaching frameworks across all domains. 
    $s(w)$ is the size of a witness $w$; 
   $i(w)$ is the index of witness $w$ in the order $\precdot_W$. * For  \optb and small-P3 `Reps. taught' not known exactly.
     }
    \label{table-small-results-domains}

    \vspace{0.25in}
    \centering
    \begin{tabular}{|l|r|r|r|}
    \hline
        \multicolumn{1}{|p{0.4\linewidth}|}{\centering \multirow{2}{2em}{Domain}}
        & \multicolumn{1}{|p{0.07\linewidth}|}{\centering \multirow{2}{4em}{Witness}}
        & \multicolumn{1}{|p{0.17\linewidth}|}{\centering $\%$ Witness Index Lower}
        & \multicolumn{1}{|p{0.15\linewidth}|}{\centering $\%$ Witness Size Smaller}
        \\ \hline
         3-DNF & $|w|\leq5$ &  94.98\% & 94.06\% \\
         \hline
        3-term DNF & $|w|\leq5$ &  97.65\% & 97.06\% \\
         \hline
        \multirow{2}{14em}{3-Term DNF with permutations}
         & $|w|\leq5$ &  90.59\% & 88.24\% \\ 
        \hhline{|~|-|-|-|} 
         & $|w|=5$  &  40.00\% & 30.00\% \\
         \hline
        \multirow{1}{19em}{3-Term DNF w/ permuts. and duplicates} 
         & $|w|\leq5$ &  84.12\% & 81.18\% \\
         \hline
        P3 & $s(w) \leq 6$ & 13.98\% & 6.55\% \\
        \hline
        small-P3 & $s(w) \leq 4$ & 24.53\% & 18.87\% \\
        \hline
   \end{tabular}
    \caption{Greedy beating Eager: Common concepts with earlier (lower index) and simpler (smaller size) witness.
    }
    \label{table-small-results-greedy_vs_eager}
    \vspace{-0.07in}
\end{table*}


\subsection{Universal Language}
\label{sec-exper-setup-p3}

We also analyze the new teaching frameworks for the Turing-complete language P3. 
P3 is a simple string manipulation language 
inspired by P", 
introduced by
Corrado B\"{o}hm  (\cite{bohm1964family}), 
the first GOTO-less imperative language proved Turing-complete, i.e., universal. The most popular variant (Brainfuck) 
has 8 instructions in total. 
We employ another version called P3, also universal and having just 7 instructions: {\tt $<>$+-[]o}. We consider P3-programs that take binary inputs and generate binary outputs.
A representation is a P3 program, such as the following one, which performs the left shift operation (e.g., on input 10010 we get output 00101): {\tt $>$[o$>$]$<$[$<$]$>$o}.
Regarding the semantics of P3, we refer to \cite{telle2019teaching}. 
%
%

A witness is a set of pairs of binary strings defining an input and an output, and we use only positive examples. For example, the witness $w=\{\langle 0100, 00\rangle, \langle001, \rangle, \langle00, 00\rangle \}$  is consistent with any program that on input 0100 outputs 00, on input 001 has no output and on input 00 outputs 00. 

$\precdot_W$: We use a simple encoding of example sets (the number of bits)
as the size function $s(w)$, and break ties in a deterministic way. 
The witness $w$ above has size $s(w)= 13$.

$\precdot_R$: We use the lexicographic order on instructions {\tt $<>$+-[]o} to define the $\precdot_R$ length-lexicographic ordering of P3 programs. 

We perform two experiments, which we call P3 and small-P3. In P3 we repeat the experiment from~\cite{telle2019teaching}. Note that computing equivalence classes of 
programs is undecidable in general, and
we cannot decide if a program enters an infinite loop. We 
limit the number of timesteps for the computation of each program before breaking. In the P3 experiment, the set of representations/programs $R$ is a potentially infinite set, but Eager and Greedy end up never going beyond the $1.9*10^9$ first programs.
The witness set $W$ contains all sets of size at most 6.
Since we cannot compute the consistency graph and run any version of Optimal, we
conduct a second experiment, small-P3.

In small-P3 we start with all self-congruent
witnesses of size at most 4, and the first 10,000 programs in $\precdot_R$. We filter out every witness without consistent program and every program without consistent witness.
A set of 1,267 programs is then $R$, and a set of 260 witnesses is $W$. We compute the consistency graph $G_R$ and use $R/W$ as the concept class, i.e., considering two programs equivalent if they are consistent with the exact same subset of $W$. 

\subsection{Experimental Results}
\label{sec-exper-results}

In Table~\ref{table-small-results-domains} we present an overview of selected results for all the experimental domains described in subsections~\ref{sec-exper-setup-boolean} and~\ref{sec-exper-setup-p3}. 
Table~\ref{table-small-results-greedy_vs_eager} reports on \% of cases where Greedy is able to teach concepts with earlier and simpler witnesses than those used by Eager, and Figure \ref{fig-preformance_vs_red_spread} shows that Redundancy Spread to a large degree explains this behavior.



As Table~\ref{table-small-results-domains} shows, for most of the Boolean experiments, Greedy manages to teach as many representations as the amount of witness sets available.  
This aligns with the suitability of Greedy for teaching in a scenario of redundant representations.
Moreover, we can observe that Greedy is always able to teach more concepts than a method like Eager that specializes in teaching unique concepts.
Since Greedy teaches many more concepts and representations, sometimes it does so by making use of witnesses that are larger or later in the order.
Yet, even in scenarios without redundancy, such as 3-DNF, where Eager might be considered more suitable, Greedy still teaches some more concepts, and does it with earlier and smaller witnesses.
 Greedy performs almost as good as \opta and \optb in terms of the number of representations taught, and the largest size and highest index of witnesses used. However, \optb (which is used rather than \opta whenever not all of $R$ can be covered) always covers a higher amount of concepts than Greedy.
The correlation between the percentage of cases where Greedy outperforms Eager and the Redundancy Spread is shown  in Figure~\ref{fig-preformance_vs_red_spread}, and confirms the latter as a strong metric for predicting these performance comparisons. A more detailed analysis can be found in Supplement. 


Analyzing now the universal language domain, we see that
Greedy allows to teach over three times more (almost five for small-P3) programs than Eager. 
For P3 we do not have the consistency graph and thus cannot compute any version of Optimal. Quite remarkably,  in Figure \ref{fig:2scat-bits-greedy}  we see that for P3 programs the witness sets of Greedy are usually smaller than the programs they identify. A similar behavior for fewer programs was noted for Eager in \cite{telle2019teaching}.
For small-P3 we use $R/W$ as a proxy for $C$ and note that \optb is able to cover all these 106 concepts, but seemingly at the price of fewer representations (214 versus 225 for Greedy). However, note that our implementation of \optb is actually only able to optimize the number of concepts covered. The further maximization of representations that \optb requires may in fact be NP-hard, but we must leave this as an open problem.
On concepts common to Eager and Greedy, we have the latter finding smaller witnesses on a relatively small percentage of the total for Small-P3 (see Figure \ref{fig-preformance_vs_red_spread}) but even more so for P3. For P3 we do not have the consistency graph, so we cannot compute Redundancy Spread, but we hypothesize that the small values of 13.98\% and 6.55\%, still in favor of Greedy, which we observe for P3, are due to an even lower Redundancy Spread 
of its underlying consistency graph. 


\begin{figure}
    \centering
    \includegraphics[width=0.9\linewidth]{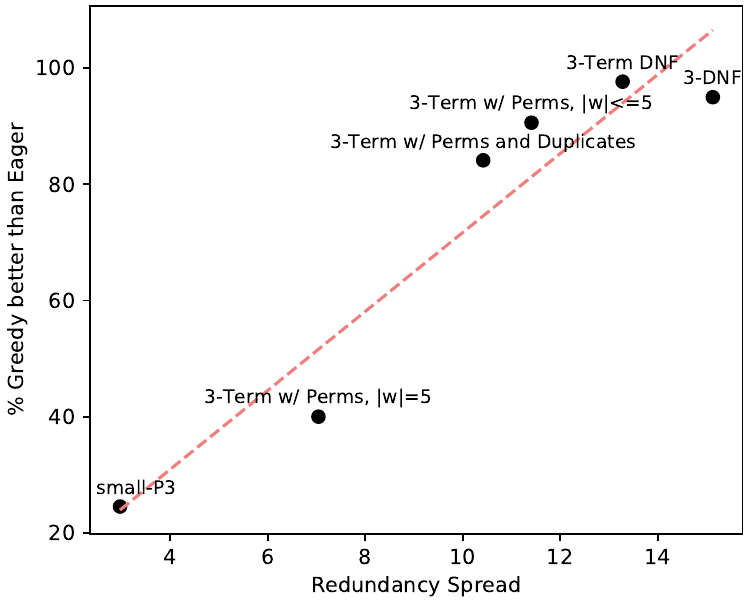}
    \vspace{-0.05in}
    \caption{Relation between {\em Redundancy Spread}, and {\em \% Greedy better than Eager}. Data from Tables \ref{table-small-results-overview_graph} and \ref{table-small-results-greedy_vs_eager}. The red line is the best fit linear approximation.
    }.   
    \label{fig-preformance_vs_red_spread}
  \vspace{-0.05in}
\end{figure}

\begin{figure}[t] 
\centering
\includegraphics[width=0.95\linewidth]{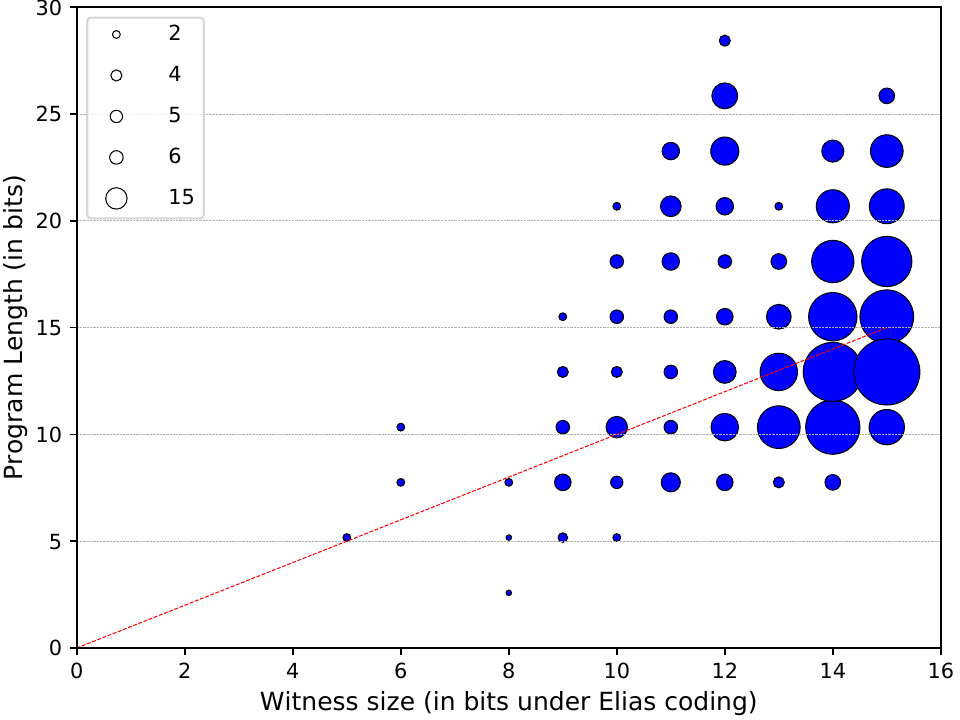} 
\vspace{-0.05in}
\caption{Greedy: Program length versus witness size, using Elias coding (\cite{DBLP:journals/tit/Elias75}). Circles above the unit diagonal denote witness smaller than program, with size of circle = number of programs.}
\label{fig:2scat-bits-greedy}
\vspace{-0.05in}
\end{figure}

\section{Discussion}
\label{sec-disc}
The notion of redundancy we explore in this paper is very old, dating back to numeral systems (e.g., IIII and IV being two alternative representations of the number 4 in Roman numerals). Leibniz was interested in prime decomposition as a tool for his {\em characteristica universalis} in which the representations and concepts could have a bijective mapping. 
In numeration systems, this is possible. Indeed, it was B\"ohm himself (the 'father' of P3) who proved that every natural number has a unique representation in bijective base-$k$ ($k \geq 1$) \cite{bohm1964family}. 
However, we also know that in many other languages, such as Turing-complete languages,  this is not possible.


The more realistic view of teaching adopted in this paper modifies some key elements of the traditional machine teaching setting and gives more relevance to the representation language. Given the high predictiveness for some teaching indicators, the new metric of redundancy spread introduced in this paper, could be used to anticipate what language representations are better than others for teaching and learning. The size of representations becomes a natural way of imposing an order on them, which can be used in the teaching protocols. Of course, all this comes at a cost, and some of the protocols seen here, such as Greedy, are computationally expensive. However, none of them requires the calculation of equivalence classes, which may also be expensive (or undecidable) for some formal languages. The most surprising finding is that a protocol such as Greedy that is designed to teach all representations ends up teaching more concepts than Eager, which aims at only teaching concepts. The experimental and theoretical results indicate that there is a wide spectrum between these two protocols, and future work could explore some other teaching protocols for representations that are somewhat in between these two, or have better theoretical or statistical properties. 







\FloatBarrier  

\bibliography{main}

\begin{thebibliography}{10}

\bibitem{alpuente2010compact}
Mar{\'\i}a Alpuente, Marco Comini, Santiago Escobar, Moreno Falaschi, and Jos{\'e} Iborra.
\newblock A compact fixpoint semantics for term rewriting systems.
\newblock {\em Theoretical Computer Science}, 411(37):3348--3371, 2010.

\bibitem{balbach2008measuring}
Frank~J Balbach.
\newblock Measuring teachability using variants of the teaching dimension.
\newblock {\em Theoretical Computer Science}, 397(1-3):94--113, 2008.

\bibitem{bohm1964family}
Corrado B{\"o}hm.
\newblock On a family of {T}uring machines and the related programming language.
\newblock {\em ICC Bulletin}, 3(3):187--194, 1964.

\bibitem{DBLP:journals/tit/Elias75}
Peter Elias.
\newblock Universal codeword sets and representations of the integers.
\newblock {\em {IEEE} Trans. Inf. Theory}, 21(2):194--203, 1975.

\bibitem{enflo1994sparse}
Per Enflo, Andrew Granville, Jeffrey Shallit, and Sheng Yu.
\newblock On sparse languages l such that ll= $\sigma$.
\newblock {\em Discrete Applied Mathematics}, 52(3):275--285, 1994.

\bibitem{no-clash}
Shaun Fallat, David Kirkpatrick, Hans~U Simon, Abolghasem Soltani, and Sandra Zilles.
\newblock On batch teaching without collusion.
\newblock {\em Journal of Machine Learning Research}, 24:1--33, 2023.

\bibitem{DBLP:journals/corr/abs-2309-16797}
Chrisantha Fernando, Dylan Banarse, Henryk Michalewski, Simon Osindero, and Tim Rockt{\"{a}}schel.
\newblock Promptbreeder: Self-referential self-improvement via prompt evolution.
\newblock {\em CoRR}, abs/2309.16797, 2023.

\bibitem{DBLP:conf/ace/Finnie-AnsleyDB22}
James Finnie{-}Ansley, Paul Denny, Brett~A. Becker, Andrew Luxton{-}Reilly, and James Prather.
\newblock The robots are coming: Exploring the implications of openai codex on introductory programming.
\newblock In Judy Sheard and Paul Denny, editors, {\em {ACE} '22: Australasian Computing Education Conference, Virtual Event, Australia, February 14 - 18, 2022}, pages 10--19. {ACM}, 2022.

\bibitem{gao2017preference}
Ziyuan Gao, Christoph Ries, Hans~Ulrich Simon, and Sandra Zilles.
\newblock Preference-based teaching.
\newblock {\em Journal of Machine Learning Research}, 18:31:1--31:32, 2017.

\bibitem{goldman1995complexity}
Sally~A Goldman and Michael~J Kearns.
\newblock On the complexity of teaching.
\newblock {\em Journal of Computer and System Sciences}, 50(1):20--31, 1995.

\bibitem{hadley1999language}
Robert~F Hadley and Vlad~C Cardei.
\newblock Language acquisition from sparse input without error feedback.
\newblock {\em Neural Networks}, 12(2):217--235, 1999.

\bibitem{Hvardstun2023XAIWM}
Brigt Arve~Toppe H{\aa}vardstun, C.~Ferri, Jos{\'e} Hern{\'a}ndez-Orallo, Pekka Parviainen, and Jan~Arne Telle.
\newblock {XAI} with machine teaching when humans are (not) informed about the irrelevant features.
\newblock In {\em ECML/PKDD}, 2023.

\bibitem{muggleton1994inductive}
Stephen Muggleton and Luc De~Raedt.
\newblock Inductive logic programming: Theory and methods.
\newblock {\em The Journal of Logic Programming}, 19:629--679, 1994.

\bibitem{bergadanodeclarative}
Claire Nédellec, Céline Rouveiro, Hilde Adé, Francesco Bergadano, and Birgit Tausend.
\newblock Declarative bias in {ILP}.
\newblock 1995.

\bibitem{SHAFTO201455}
Patrick Shafto, Noah~D. Goodman, and Thomas~L. Griffiths.
\newblock A rational account of pedagogical reasoning: Teaching by, and learning from, examples.
\newblock {\em Cognitive Psychology}, 71:55 -- 89, 2014.

\bibitem{telle2019teaching}
Jan~Arne Telle, Jos{\'e} Hern{\'a}ndez-Orallo, and C{\`e}sar Ferri.
\newblock The teaching size: computable teachers and learners for universal languages.
\newblock {\em Machine Learning}, 108(8-9):1653--1675, 2019.

\bibitem{whigham1996search}
Peter~A Whigham.
\newblock Search bias, language bias, and genetic programming.
\newblock {\em Genetic Programming}, 1996:230--237, 1996.

\bibitem{yang2021mitigating}
Scott Cheng-Hsin Yang, Wai~Keen Vong, Ravi~B Sojitra, Tomas Folke, and Patrick Shafto.
\newblock Mitigating belief projection in explainable artificial intelligence via bayesian teaching.
\newblock {\em Scientific reports}, 11(1):1--17, 2021.

\bibitem{yu1988can}
Sheng Yu.
\newblock Can the catenation of two weakly sparse languages be dense?
\newblock {\em Discrete applied mathematics}, 20(3):265--267, 1988.

\bibitem{zhu2018overview}
Xiaojin Zhu, Adish Singla, Sandra Zilles, and Anna~N Rafferty.
\newblock An overview of machine teaching.
\newblock {\em arXiv preprint arXiv:1801.05927}, 2018.

\bibitem{zilles2011models}
Sandra Zilles, Steffen Lange, Robert Holte, and Martin Zinkevich.
\newblock Models of cooperative teaching and learning.
\newblock {\em Journal of Machine Learning Research}, 12(Feb):349--384, 2011.

\end{thebibliography}

\end{document}